\documentclass{article} 
\pdfoutput=1
\usepackage{aaai,times}
\usepackage{hyperref}
\usepackage{url}
\usepackage{amsthm}
\usepackage{graphicx}
\usepackage{subcaption}
\usepackage{todonotes}
\usepackage{algorithm2e}
\usepackage{booktabs}
\DontPrintSemicolon

\newcommand{\citet}[1]{\citeauthor{#1}~\shortcite{#1}}
\newcommand{\citep}[1]{\cite{#1}}

\newtheorem{mydef}{Definition}
\newtheorem{mytheorem}{Theorem}

\title{Constructing Abstraction Hierarchies Using a Skill-Symbol Loop}

\author{
George Konidaris \\
Departments of Computer Science and Electrical \& Computer Engineering\\
Duke University, Durham NC 27708 \\
\texttt{gdk@cs.duke.edu} 
}

\begin{document}

\maketitle

\begin{abstract}
\begin{quote}
We describe a framework for building abstraction hierarchies whereby
an agent alternates skill- and representation-acquisition phases to construct
a sequence of increasingly abstract Markov decision processes. Our formulation
builds on recent results showing that the appropriate abstract representation of a problem
is specified by the agent's skills. We describe how such a hierarchy can be used
for fast planning, and illustrate the construction of an appropriate hierarchy for the Taxi domain.
\end{quote}
\end{abstract}

\section{Introduction}

One of the core challenges of artificial intelligence is that of linking abstract decision-making
to low-level, real-world action and perception. 
Hierarchical reinforcement learning methods \cite{Barto03} approach this problem through the
 use of 
high-level 
 temporally extended macro-actions, or skills, which can significantly 
decrease planning times \cite{Sutton99}. 
Skill acquisition (or skill discovery) algorithms (recently surveyed by \citet{Hengst12}), 
aim to discover appropriate high-level skills autonomously.
However, in most hierarchical reinforcement learning research the state space does not 
change once skills have been acquired. 
An agent that has acquired high-level skills must still plan in its original low-level
state space---a potentially very difficult task when that space is high-dimensional and continuous. 
Although some of the earliest formalizations of hierarchical reinforcement learning \cite{Parr97,Dietterich00} featured 
hierarchies where both the set of available actions and the state space changed with the level
of the hierarchy, there has been almost no work on automating the representational aspects of 
such hierarchies. 

Recently, \citet{Konidaris14} considered the question of how 
to construct a symbolic representation 
suitable for  planning in high-dimensional continuous domains, given a set of high-level skills. 
The key result of that work was that the appropriate abstract representation of the
problem was directly determined by characteristics of the skills available to the agent---the skills
determine the representation, and adding new high-level skills must result in a new representation. 
 
We show that these two processes can be combined into a \textit{skill-symbol loop}: the agent
acquires a set of high-level skills, then constructs the appropriate representation for planning using them,
resulting in a new problem in which the agent can again perform skill acquisition. Repeating this process
leads to a true abstraction hierarchy where both the available skills and the state space become more
abstract at each level of the hierarchy. 
We describe the properties of the
resulting abstraction hierarchies and demonstrate the construction and use of one such hierarchy in the
Taxi domain. 

\section{Background}

Reinforcement learning problems are typically formalized as 
\textit{Markov decision processes} or MDPs, represented
by a tuple $M = (S, A, R, P, \gamma)$, where $S$ is a set 
of states, $A$ is a set of actions, $R(s, a, s')$ is the reward the
agent receives when executing action $a$ in state $s$ and
transitioning to state $s'$, $P(s' | s, a)$ is the probability
of the agent finding itself in state $s'$ having executed action $a$ in
state $s$, and $\gamma \in (0, 1]$ is a discount factor. 

We are interested in the multi-task reinforcement learning setting where, rather 
than solving a single  MDP, the agent is tasked
with solving several problems drawn from some task distribution. Each
individual problem is obtained by adding a set of start and goal states to a \textit{base MDP} that specifies 
the state and action spaces and background reward function. The agent's task is to 
 minimize the average
time required to  solve new problems
drawn from this distribution.

\subsection{Hierarchical Reinforcement Learning}

Hierarchical reinforcement learning \cite{Barto03} 
is a framework for learning and planning using higher-level actions built
out of the primitive actions available to the agent. Although other
formalizations exist---mostly notably the MAX-Q \cite{Dietterich00} and 
Hierarchy of Abstract Machines \cite{Parr97} approaches---we adopt the
\textit{options framework} \cite{Sutton99}, which models temporally abstract
macro-actions as \textit{options}.

An option $o$ consists of three components: an \textit{option policy}, $\pi_o$,
which is executed when the option is invoked; an \textit{initiation set}, 
$I_o = \{s | o \in O(s)\}$, which describes the states in which the option may
be executed; and a \textit{termination condition}, $\beta_o(s) \rightarrow [0, 1]$, 
which describes the probability that an option will terminate upon reaching state $s$.

An MDP where primitive actions are replaced by a set of possibly temporally-extended
options (some of which could simply
execute a single primitive action) is known as a \textit{semi Markov decision process} (or SMDP),
which generalizes MDPs to handle action executions that may take
more than one time step. 
An SMDP is described by a tuple $M = (S, O, R, P, \gamma)$, where $S$ is a set of states;
$O$ is a  set of options; 
$R(s', \tau | s, o)$ is the reward received when executing option $o \in O(s)$
at state $s \in S$, and arriving in state $s' \in S$ after $\tau$ time steps; 
$P(s', \tau| s, o)$ is a PDF describing the
probability of arriving in state $s' \in S$, $\tau$  time steps after executing option
 $o \in O(s)$ in state
$s \in S$; and $\gamma \in (0, 1]$ is a discount factor, as before. 
 
The problem of deciding which options an agent should acquire is known as the 
\textit{skill discovery problem}.  
A skill discovery algorithm must, through experience (and perhaps
additional advice or domain knowledge), acquire new options by specifying their initiation set, $I_o$, 
and termination condition, $\beta_o$. The option policy is usually specified
indirectly via  an option reward function, $R_o$, which is used to learn $\pi_o$. Each new skill is  
added to the set of options available to the agent with the aim of either solving the original
or subsequent tasks more efficiently. Our framework is agnostic to the specific
skill discovery method used (many exist). 

\subsection{Representation Acquisition}

While skill acquisition allows an agent to construct higher-level actions, it alone is insufficient
for constructing a true abstraction hierarchy because the agent must still plan in the original state space, no matter how abstract its actions become.
A complementary approach is taken by recent work
on \textit{representation acquisition} 
\cite{Konidaris14}, which considers
the question of constructing a symbolic description of an SMDP suitable for high-level planning.
Key to this is the definition of a symbol as a name referring to a set of states:
\begin{mydef}
A  propositional symbol $\sigma_Z$ is the
name of a test $\tau_Z$,
and corresponding set of states $Z = \{ s~\in S~|~\tau_Z(s)~=~1\}$.
\end{mydef}

The test, or \textit{grounding classifier}, is a compact representation of a (potentially
uncountably infinite) set of states (the \textit{grounding set}). Logical operations (e.g., \texttt{and}) using the resulting symbolic names have the semantic meaning of set operations (e.g., $\cap$) over the grounding sets,
which allows us to reason about which symbols (and corresponding grounding classifiers)
an agent should construct in order to be able to determine
the feasibility of high-level plans composed of sequences of options. We use the \textit{grounding operator}
$\mathcal{G}$ to obtain the grounding set of a symbol or symbolic expression; for example,
$\mathcal{G}(\sigma_Z) = Z$, $\mathcal{G}(\sigma_A \textrm{ and } \sigma_B) = A \cap B$.
For convenience we also define $\mathcal{G}$ over collections of symbols; for a
set of symbols $A$, we define 
$\mathcal{G}(A) = \cup_i \mathcal{G}(a_i), \forall a_i \in A$. 

\citet{Konidaris14} 
showed that defining a symbol for each option's initiation set and the symbols necessary
to compute its image (the set of states the agent might be in after executing the option
from some set of starting states) are \textit{necessary and sufficient for planning using
that set of options}. The feasibility
of a plan is evaluated by computing  each successive option's image, 
and then testing whether it is a subset of the next option's initiation set.
Unfortunately, computing the image of an option is intractable in the general case. 
However, the definition of the image  for two common classes of options is both natural and computationally
very simple. 

The first is the subgoal option: the option
reaches some set of states and terminates, and the state  it terminates in can be considered independent
of the state execution began in. In this case we can create a symbol for that set (called the \textit{effect set}---the
set of all possible states the option may terminate in),
and use it directly as the option's image. We thus obtain $2n$ symbols for $n$ options (a symbol for each option's
initiation and effect sets), from which we can build a  
\textit{plan graph} representation: a graph with $n$ nodes, and an edge from node $i$ to node $j$ if 
option $j$'s initiation set is a superset of option $i$'s effect set. Planning amounts to finding a path in the plan graph;
once this graph has been computed, the grounding
classifiers can be discarded.

The second class of options are \textit{abstract subgoal} options: the low-level
state is factored, and some variables are set to a subgoal (again, independently of the starting state) while others
remain unchanged. The image operator can then be computed using the intersection of the effect set (as in the subgoal
option case) and the starting state classifier with the modified factors projected out.
This results in a STRIPS-like factored representation which can be automatically converted to PDDL
\cite{McDermott98} and used as input to an off-the-shelf task planner. After this conversion 
the grounding classifiers can again be discarded.

\section{Constructing Abstraction Hierarchies}

\begin{figure*}[!!ht]
\centering
\subcaptionbox{}[0.24\linewidth][l]
{
\includegraphics[width=0.15\linewidth]{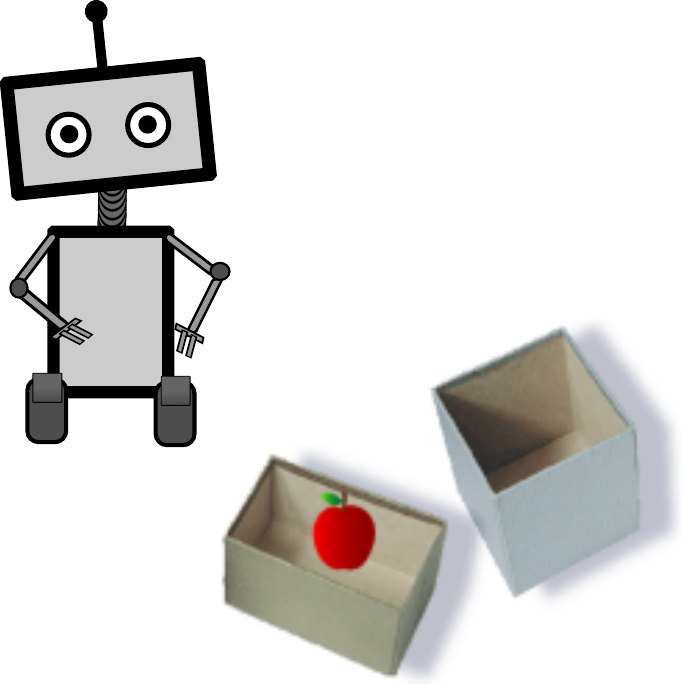}
}
\subcaptionbox{}[0.2\linewidth][l]
{
\begin{tabular}{l}
$S1:$\\
\textit{above-box-1} $\times$\\
\textit{above-box-2} $\times$\\
\textit{pregrasped} $\times$\\
\textit{grasped} $\times$\\
\textit{apple-in-box-1} $\times$\\
\textit{apple-in-box-2} \\
\end{tabular}
}
\subcaptionbox{}[0.2\linewidth]
{
\begin{tabular}{l}
$S2:$\\
\textit{grabbed} $\times$\\
\textit{above-box-1} $\times$\\
\textit{above-box-2} $\times$\\
\textit{apple-in-box-1} $\times$\\
\textit{apple-in-box-2} \\
\\
\end{tabular}
}
\subcaptionbox{}[0.2\linewidth]
{
\begin{tabular}{l}
$S_3:$\\
\textit{apple-in-box-1} $\times$\\
\textit{apple-in-box-2}\\
\\
\\
\\
\\
\end{tabular}
}

\caption{A robot must move an apple between two boxes (a). Given a set of motor primitives it can
form a discrete, factored state space (b). Subsequent applications of skill acquisition result in successively more
abstract state spaces (c and d).}
\label{RobotAppleFig}
\end{figure*}

These results show that the two fundamental aspects of hierarchy---skills and 
representations---are tightly coupled: skill acquisition drives representational 
abstraction.  An agent that has performed skill acquisition in an MDP to obtain
higher-level skills can automatically determine a new abstract state representation suitable
for planning in the resulting SMDP.  We now show that these two processes can be alternated
to construct an abstraction hierarchy.

We assume the following setting: an agent is faced with some base MDP
$M_0$, and aims to construct an abstraction hierarchy that
enables efficient planning for new problems posed in $M_0$, each of which is 
specified by a start
and goal state set.  $M_0$ may be continuous-state and
even continuous-action, but  all subsequent levels of the
hierarchy will be constructed to be discrete-state and discrete-action. We adopt the following
 definition of an abstraction hierarchy: 
\begin{mydef}
An $n$-level hierarchy on base MDP $M_0 = \left(S_0, A_0, R_0, P_0 \right)$
is a collection of MDPs $M_i = \left(S_i, A_i, R_i, P_i \right)$, $i \in \{1, ..., n\}$,
such that each action set $A_j$, $0 < j \leq n$, is a set of options defined over
$M_{j-1}$ (i.e., $M_{{j-1}^+} = \left(S_{j-1}, A_j, R_{j-1}, P_{j-1}\right)$ is an SMDP). 
\label{HierarchyDef}
\end{mydef}

This captures the core assumption behind  hierarchical reinforcement
learning: hierarchies are built through macro-actions. Note that this formulation
 retains the downward refinement property from classical hierarchical
planning \cite{Bacchus91}---meaning that a plan at level $j$
can be refined to a plan at level $j-1$ without backtracking to level $j$ or higher---because a policy at any level 
is also a (not necessarily Markovian \cite{Sutton99}) policy at any
level lower, including $M_0$.
However, while 
Definition \ref{HierarchyDef} links the action set of each MDP to the action set of its predecessor in the hierarchy,
it says nothing about how to link their state spaces. 
To do so, we must in addition determine how to construct a new state space $S_{j}$,
transition probability function $P_{j}$, and reward function $R_j$.

Fortunately, this is exactly what representation acquisition provides: a method for
constructing 
a new symbolic representation suitable for planning in $M_{{j-1}^+}$ using the options in $A_j$. 
This provides a new state space $S_{j}$, which, combined with $A_j$, specifies
$P_j$. The only remaining component is the reward function.
A representation construction algorithm based on sets \cite{Konidaris14}---such as we adopt here---is 
insufficient for reasoning about expected rewards,
which requires a formulation based on distributions \cite{Konidaris15}. For simplicity, we can remain consistent
and simply set the reward to a uniform transition penalty of $-1$; alternatively, we can adopt just one 
aspect of the distribution-based representation
and set $R_j$ to the empirical mean of the rewards obtained when executing each  option.

Thus,
we have all the components required to build level $j$ of the hierarchy from level $j-1$. This procedure
can be repeated in a \textit{skill-symbol loop}---alternating skill acquisition and representation acquisition
phases---to construct an abstraction hierarchy.
 It is important to note that there are no degrees of
freedom or design choices in the representation acquisition phase of the skill-symbol 
loop; the  algorithmic questions reside solely  in determining which skills to acquire
at each level.

This construction results in  a specific relationship between MDPs in a hierarchy:  every state at level $j$ 
refers to a set of states
at level $j-1$.\footnote{Note that  $S_{j+1}$ is not necessarily a \textit{partition} of $S_{j}$---the grounding sets of
two states in $S_{j+1}$ may overlap.} 
A grounding in $M_0$ can therefore be computed for any state at level $j$ in the hierarchy by 
applying the grounding operator $j$ times. If we denote this ``final grounding'' operator as $\mathcal{G}_0$, then 
$\forall j, s_j \in S_j, \exists Z_0 \subseteq S_0$ such that $\mathcal{G}_0(s_j) = Z_0$. 

We now illustrate the construction of an abstraction hierarchy via an example---a very simple
task that must be solved by a complex agent. 
Consider a robot in a room with two boxes, one containing an apple (Figure \ref{RobotAppleFig}a).
The robot must occasionally move the apple from one box to the other. Directly accomplishing
this involves solving a high-dimensional motion planning problem, so instead the robot is given five
motor skills: \textit{move-gripper-above1} and \textit{move-gripper-above2} use 
motion planning to move the robot's
gripper above each box; \textit{pregrasp} controls the
 gripper so that it cages the apple,
and is only executable from above it; \textit{grasp} can be executed following
\textit{pregrasp}, and runs a gradient-descent based controller to achieve wrench-closure on the apple; 
and \textit{release} drops the apple. These form $A_1$, the actions in the first level of the hierarchy,
and since they are  abstract subgoal options the robot  automatically constructs
a factored state space (see Figure \ref{RobotAppleFig}b) 
that specifies $M_2$. This enables abstract planning---the state space is independent of the complexity of the robot,
although $S_2$  contains some low-level details (e.g., \textit{pregrasped}).

Applying a skill discovery algorithm in $M_2$, the robot detects that 
\textit{pregrasp} is always followed by \textit{grasp}, and therefore replaces these actions with
 \textit{grab-apple}, which together with the remaining skills in $A_1$ forms $A_2$.
This results in a smaller MDP, $M_2$ (Figure \ref{RobotAppleFig}c), which is a good
abstract model of the task. Applying a skill
discovery algorithm to $M_2$ creates a skill that picks up the apple in whichever box it is in, and
moves it over the other box. $A_3$ now consists of just a single action, \textit{swap-apple},  requiring
 just two propositions to define $S_3$: \textit{apple-in-box-1}, and \textit{apple-in-box-2} (Figure \ref{RobotAppleFig}d). 
 The abstraction hierarchy
 has abstracted away the details of the robot (in all its complexity) and 
 exposed the (almost trivial) underlying task structure. 

\section{Planning Using an Abstraction Hierarchy}

Once an agent has constructed an abstraction hierarchy, it must be able to
use it to rapidly find plans for new problems. We formalize this process as the agent posing a \textit{plan query} to the hierarchy,
which should then be used to generate a plan for solving the problem described by the query. 
 We 
adopt the following  definition of a 
plan query:
\begin{mydef}
A \emph{plan query} is a tuple $(B, G)$, where $B \subseteq S_0$ is the set of base MDP states from which execution may begin, and 
$G \subseteq S_0$ (the goal) is the set of base MDP states in which the agent wishes to find itself following execution.
\end{mydef}

The critical question is \textit{at which level} of the hierarchy planning should take place.
We first define a useful predicate, planmatch, which determines whether an agent should
attempt to plan at level $j$ (see Figure \ref{SelectLevelFig}): 
\begin{mydef} 
A pair of abstract state sets $b$ and $g$  match a plan query $(B, G)$
(denoted \emph{planmatch$(b, g, B, G)$}) when $B \subseteq \mathcal{G}_0(b)$ and 
$\mathcal{G}_0(g) \subseteq G$.
\end{mydef}
\begin{mytheorem}
A plan can be found to solve plan query $(B, G)$ at level $j$ iff 
 $\exists b, g \subseteq S_j$ such that planmatch$(b, g, B, G)$, 
 and there is a feasible plan in $M_i$ from every state in $b$ to some state in $g$.   
 \label{MainTheorem}
 \end{mytheorem}
\begin{proof}
The MDP at level $j$ is constructed such that a plan $p$ starting from any state in $\mathcal{G}(b)$ (and hence also $\mathcal{G}_0(b)$) is guaranteed
to leave the agent in a state in $\mathcal{G}(g)$ (and hence also $\mathcal{G}_0(g)$) iff $p$ is a plan in MDP $M_j$ from $b$ to $g$ \cite{Konidaris14}. 

Plan $p$ is additionally valid from $B$ to $G$ iff $B \subseteq \mathcal{G}_0(b)$ (the start state at level $j$ refers to a set that includes all query start states) and $\mathcal{G}_0(g) \subseteq G$ (the  query goal includes all states referred to by the goal at level $j$).
\qedhere
\end{proof}

\begin{figure}[!!!ht]
\centering
\vspace{-10pt}
\includegraphics[width=0.8\linewidth]{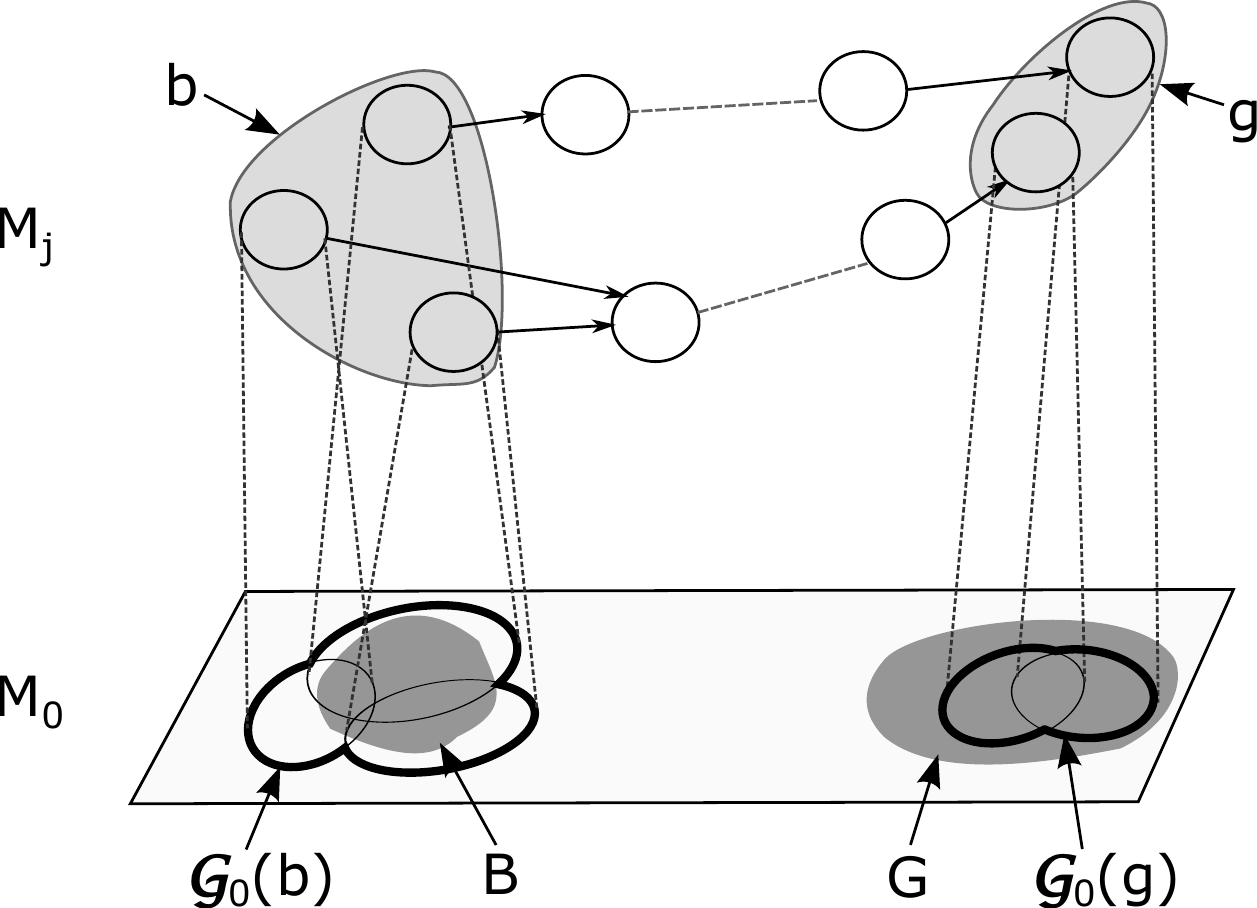}
\caption{The conditions under which a plan at MDP $M_j$ answers 
a plan query with start state set $B$ and goal state set $G$ in the base MDP $M_0$. 
A pair of state sets $b, g \subseteq S_j$ are required such that
 $B \subseteq \mathcal{G}_0(b)$,  $\mathcal{G}_0(g) \subseteq G$, and a plan exists in $M_j$
 from every state in $b$ to some state in $g$.}
\label{SelectLevelFig}
\end{figure}

Note that $b$ and $g$ may not be unique, even within a single level: because $S_{j}$ is not necessarily a partition of $S_{j-1}$, there may be multiple states, 
or sets of states, at each level whose final groundings are included by $G$ or include $B$;
a solution from \textit{any} such $b$ to any such $g$ is
sufficient.
For efficient planning it is better for 
$b$ to be a small set to reduce the number of start states
while remaining large enough to subsume $B$; if $b = S_j$ then answering the plan query
requires a complete policy for $M_j$, rather than a plan. However, finding a minimal subset is 
computationally difficult. One approach is to build the maximal candidate set $b = \{s | \mathcal{G}_0(s) \cap B \neq \emptyset, s \in S_j\}$.
This is a superset of any start match, and a suitable one exists at this level
 if and only if $B \subseteq \cup_{s\in b}\mathcal{G}_0(s)$. 
Similarly, 
 $g$ should be maximally large (and so easy to reach)  while remaining small enough so that its grounding
 set lies within $G$. At each level $j$, we can therefore collect all states that ground out to subsets of $G$:
 $g = \{s | \mathcal{G}_0(s) \subseteq G,  s \in S_j\}$. These approximations result in a unique pair of sets of states at each level---at the
 cost of potentially including unnecessary states in each set--- and can be computed in time linear in $|S_j|$.

It follows from the state abstraction properties 
of the hierarchy that
 a planmatch at level $j$ implies
the existence of a planmatch at all levels below $j$.
\begin{mytheorem}
Given a hierarchy of state spaces $\{S_0, ..., S_n\}$ constructed as above and plan query $(B, G)$, 
if $\exists b, g \subseteq S_j$ such that planmatch$(b, g, B, G)$,
for some $j, n \geq j > 0$, then $\exists b', g'  \subseteq S_{k}$
such that  planmatch$(b', g', B, G)$, $\forall k \in \{0, ..., j-1\}$. 
\end{mytheorem}
\begin{proof}
We first consider $k = j - 1$. 
Let $b' = \mathcal{G}(b)$, and $g' = \mathcal{G}(g)$. Both are, by definition, sets of states in $S_{j-1}$.
By definition of the final grounding operator, $\mathcal{G}_0(b) = \mathcal{G}_0(b')$ and 
$\mathcal{G}_0(g) = \mathcal{G}_0(g')$, and hence
$B \subseteq \mathcal{G}_0(b')$ and $\mathcal{G}_0(g') \subseteq G$. This process can be
repeated to reach any  $k < j$.
\qedhere
\end{proof}

Any plan query therefore has a unique highest level $j$ containing a planmatch.
%
%
This leads directly to Algorithm \ref{PlanAlg}, which starts
 looking for a planmatch at the highest level of the hierarchy and proceeds
downwards; it is sound and complete by
 Theorem \ref{MainTheorem}.

\begin{algorithm}
\KwIn{MDP hierarchy $\{M_0, ..., M_n\}$, query $(B, G)$.}
\For{j $\in \{n, ..., 0\}$}{
 \For{$\forall b, g \subseteq S_j$ s.t. \emph{planmatch}($b, g, B, G$)}
  {
    $\pi \leftarrow $ findplan($M_j, b, g$)\;
     \If{$\pi \neq $ \emph{null}}
     {
      \KwRet{$(M_j, \pi)$}
     }
  }
}
\KwRet{\emph{null}}\;
\caption{A simple hierarchical planning algorithm.}
\label{PlanAlg}
\end{algorithm}

The complexity of Algorithm \ref{PlanAlg} depends on its
 two component algorithms: one used to find a planmatch, and another to attempt
to find a plan (possibly with multiple start states and goals).
We denote the complexity of these algorithms as $m(|S|)$ (linear using the approach
described above) and $p(|S|)$, for
a problem with $|S|$ states, respectively. 
The complexity of finding a plan at level $l$,
where the first match is found at level $k \geq l$, is given by
$h(k, l, M) = \sum_{a=k+1}^{n} m(|S_a|) + \sum_{b=l}^k \left[ m(|S_b|) + p(|S_b|) \right],$
for a hierarchy $M$ with $n$ levels. The first term corresponds to the search for the level with the first planmatch;
the second term for the repeated planning at levels that contain a match but not a plan (a planmatch does not necessarily
mean a plan exists at that level---merely that one could). 


\section{Discussion}

The formula for $h$ highlights the fact that hierarchies make some problems easier to solve
and others harder: in the worst case, a problem that should take $p(|S_0|)$ time---one
 only solvable via the base MDP---could instead take  $\sum_{a=0}^n \left[ m(|S_a|) + p(|S_b|) \right]$ time. A key
 question is therefore how to balance the depth of the hierarchy, the rate
at which the state space size diminishes as the level increases, which specific skills to discover
at each level, and how to control false positive plan matches,  to reduce
planning time. 

Recent work has highlighted the idea that skill discovery algorithms should
aim to reduce average planning or learning time across a target distribution of tasks
\cite{Simsek08,Solway14}. Following this logic, a hierarchy $M$ for some distribution
of over task set $T$ should be constructed so as to minimize
$\int_T  h(k(t), l(t), M) P(t)dt,$
where $k$ and $l$ now both depend on each task $t$. Minimizing this quantity
over the entire distribution seems infeasible; an acceptable substitute may be to
assume that the tasks the agent has already experienced 
are drawn from the same distribution as those it will experience in the future, 
and to construct the hierarchy that minimizes $h$ averaged over past tasks.

The form of $h$ suggests two important
principles which may aid the more direct design of 
skill acquisition algorithms. One is that deeper hierarchies are not necessarily
better; each level adds potential planning and matching costs, and must be justified by 
a rapidly diminishing state space size and a high likelihood of  solving 
tasks at that level. 
Second, false positive plan matches---when 
a pair of states that match the query  is found at some level at which a plan
cannot be found---incur a significant time penalty. The hierarchy should
therefore ideally be constructed 
 so that  every likely goal state at each level is reachable from every likely start state at that level.

An agent that generates its own goals---as a completely autonomous agent would---could 
do so by selecting an existing state from an
MDP at some level (say $j$) in the hierarchy. In that case it need not search for a matching level, and could
instead immediately plan at level $j$, though it may still need to drop to lower levels if no plan is found in $M_j$.  

\section{An Example Domain: Taxi} 

We now explain the construction and use of an abstraction hierarchy for a common hierarchical reinforcement
learning benchmark: the 
Taxi  domain \cite{Dietterich00}, depicted in Figure \ref{TaxiGraph}a. A taxi must navigate a $5 \times 5$ grid, which contains
a few walls, four depots (labeled red, green, blue, and yellow), and a passenger. The taxi may
move one square in each direction (unless impeded by a wall), pick up a passenger (when occupying the same
square), or drop off a passenger (when it has previously picked the passenger up). 
A state at base MDP $M_0$ is described by $5$ state variables: the $x$ and $y$ location
of the taxi and the passenger, and whether or not the passenger is in the taxi. This results in a 
total of $650$ states ($25 \times 25 = 625$ states for when the passenger is not in the taxi, plus
another $25$ for when the passenger is in the taxi and they are constrained to have the same 
location). 

\begin{figure}[!!!ht]
\centering
\subcaptionbox{}[0.7\linewidth]
{
\includegraphics[width=0.35\linewidth]{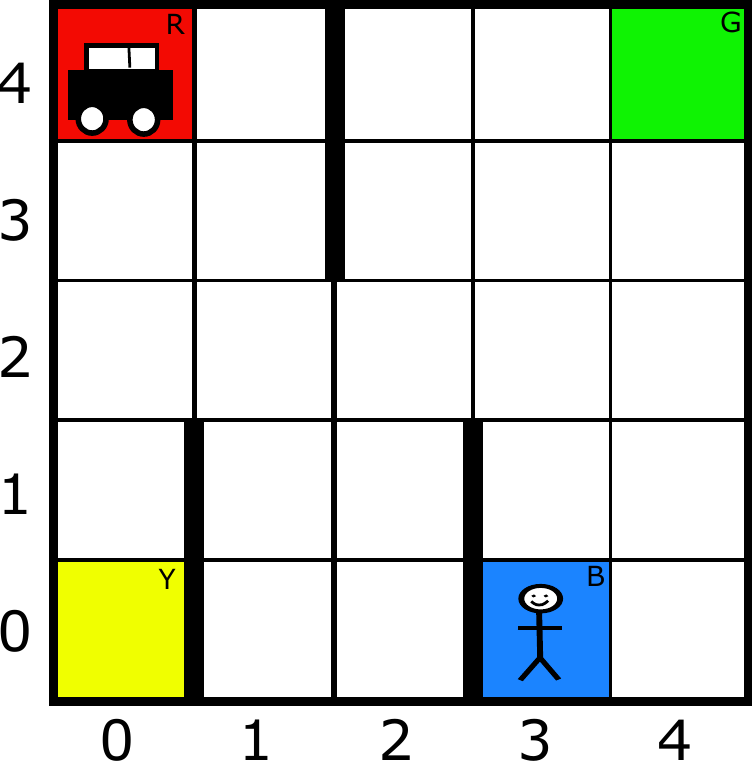}
}\\
\vspace{10pt}
\subcaptionbox{}[0.6\linewidth]
{
\includegraphics[width=0.6\linewidth, angle=90]{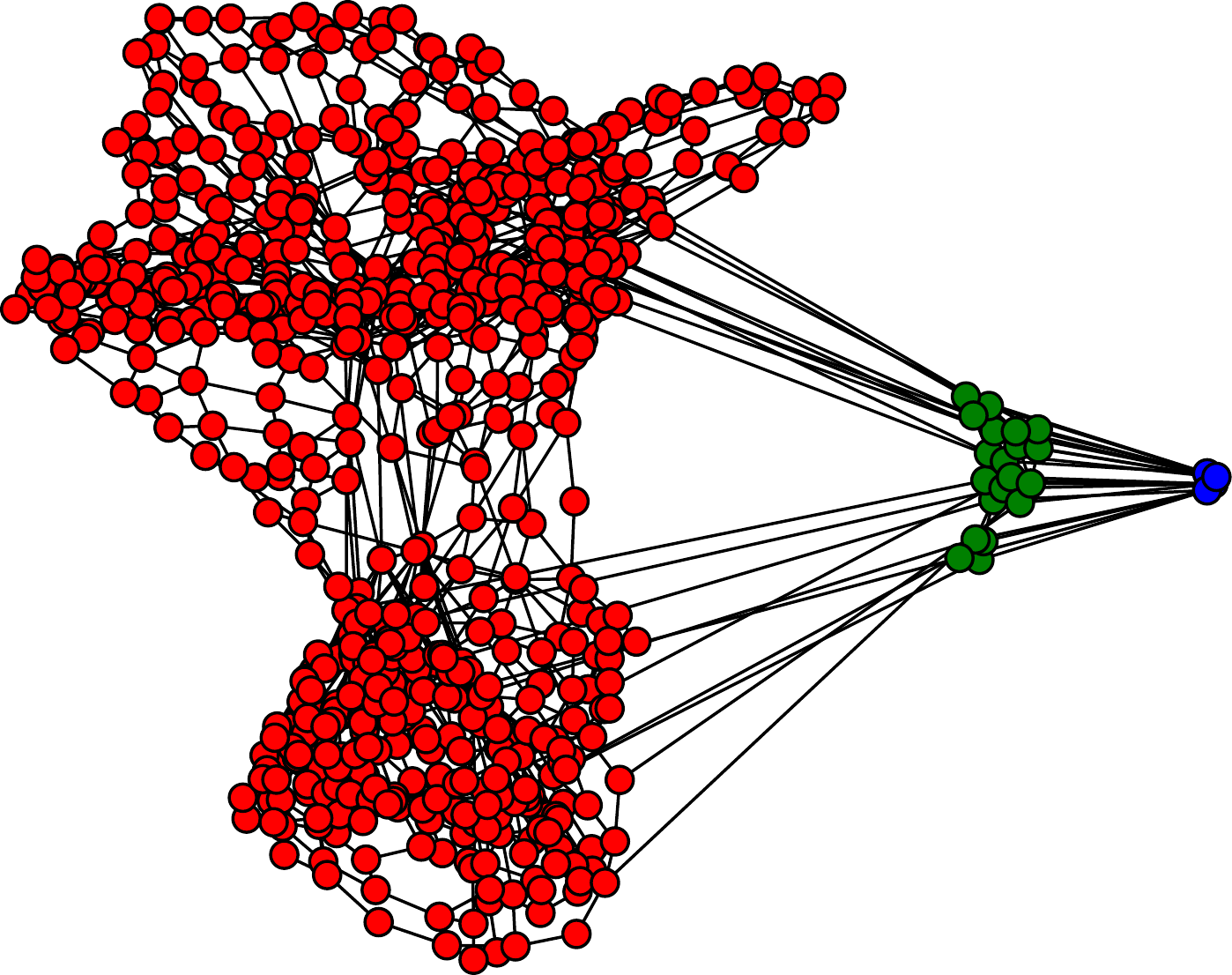}
}
\caption{The Taxi Domain (a), and its induced 3-level hierarchy. The base MDP contains $650$ states (shown in red),
which is abstracted to an MDP with $20$ states (green) after the first level of options, and one with $4$ states (blue) after the second. At the base level, the agent makes decisions about moving the taxi one step at a time; 
at the second level, about moving the taxi between depots; 
at the third, about moving the passenger between depots.}
\label{TaxiGraph}
\end{figure}

 We now describe the construction of a hierarchy for the taxi domain using hand-designed
 options at each level, and 
present some results for planning using Algorithm \ref{PlanAlg} for three example plan queries.

\begin{figure*}[!!!t]
\centering
\begin{tabular}{ccrrrrr}
\toprule
 & & \multicolumn{3}{c}{\textbf{Hierarchical Planning}} & & \\ \cmidrule(r){3-5}
\textbf{Query} & \textbf{Level} & \textbf{Matching} & \textbf{Planning} & \textbf{Total} & \textbf{Base + Options} & \textbf{Base MDP} \\
\midrule
1 & 2 & $<$1 & $<$1 & \textbf{$<$1} & 770.42 & 1423.36 \\
2 & 1 & $<$1 & 10.55 & \textbf{11.1} & 1010.85 & 1767.45 \\
3 & 0 & 12.36 & 1330.38 & 1342.74 & \textbf{1174.35} & \textbf{1314.94} \\
\bottomrule
\end{tabular}
\caption{Timing results for three example queries in the Taxi domain. 
The final three columns compare the total time for planning using the hierarchy, by planning in the SMDP obtained by adding all options into the base MDP 
(i.e., using options but not changing the representation), and by flat planning in the base MDP.
 All times are in milliseconds and are averaged over $100$ samples, obtained using a Java implementation run on a
 Macbook Air with a 1.4 GHz Intel Core i5 and 8GB of RAM.}
 \label{ResultsTab}
\end{figure*}

\paragraph{Constructing $M_1$.}
In this version of taxi, the agent is able to move the taxi to, and drop the passenger at, any square, but it expects to face a distribution of problems
generated by placing the taxi and the passenger at a depot at random, and selecting a random 
target depot at which the passenger must be deposited.
Consequently, we create navigation options for driving the taxi to each depot, and retain the existing put-down and pick-up 
options.\footnote{These roughly
correspond to the hand-designed hierarchical actions used in \citet{Dietterich00}.} 
These options over $M_0$ form the action set for level $1$ of the hierarchy: $A_1 = \{$\textit{drive-to-red}, \textit{drive-to-green}, \textit{drive-to-blue}, \textit{drive-to-yellow}, \textit{pick-up}, \textit{put-down}$\}$.

Consider the \textit{drive-to-blue-depot} option. It is executable in all states (i.e., its initiation
set is $S_0$), and terminates with the taxi's $x$ and $y$ position set to the position of the blue depot; 
if the passenger is in the taxi, their location is also set to that of the blue depot; otherwise, their location (and the
fact that they are not in the taxi) remains unchanged. It can therefore be partitioned 
into two abstract subgoal options: one, when the passenger is in the taxi, sets the $x$ and $y$ positions of the
taxi and passenger to those of the blue depot; 
another, when the passenger is not in the taxi, sets the taxi $x$ and $y$ coordinates and leaves those of the 
passenger unchanged. Both leave the \textit{in-taxi} state variable unmodified. Similarly, the \textit{put-down}
and \textit{pick-up} options are executable everywhere and when the taxi and passenger are in the same square,
respectively, and modify the \textit{in-taxi} variable while leaving the remaining variables the same. 
Partitioning all options in $A_1$ into abstract subgoal options results in 
 a factored state space consisting of $20$ reachable
states where the taxi or passenger are at the depot locations ($4 \times 4$ states for when the passenger is not in the taxi, plus $4$ for when they are).

\paragraph{Constructing $M_2$.} Given $M_1$, we now build the second level of the hierarchy by
 constructing options that pick up the passenger (wherever they are), move them
to each of the four depots, and drop them off. These options become 
$A_2 = \{$\textit{passenger-to-blue}, \textit{passenger-to-red},
\textit{passenger-to-green}, \textit{passenger-to-yellow}$\}$. Each option is executable whenever the passenger
is not already at the relevant depot, and it leaves the passenger and taxi at the depot, with the passenger 
 outside the taxi.  Since these are subgoal (as opposed to abstract subgoal) options, the resulting MDP, $M_1$, 
consists of only $4$ states (one for each location of the passenger) and is a simple (and coincidentally fully connected) graph.
The resulting hierarchy is 
depicted in
Figure \ref{TaxiGraph}b.

We used the above hierarchy to compute plans for three example queries, using dynamic programming and decision trees for planning
and grounding classifiers, respectively. The results are given in Table \ref{ResultsTab}; we next present each query, and step through the
matching process in detail.

\textbf{Example Query 1.}
Query $Q_1$ has the passenger start at the blue depot (with the taxi at an unknown depot) and request to be moved to the 
red depot. In this case $B_1$ refers to all states where the passenger is at the blue depot and the 
taxi is located at one of four depots, and $G_1$ similarly refers to the red depot. 
The agent must first determine the appropriate level to
plan at, starting from $M_2$,
the highest level of the hierarchy. It finds state $s_b$ where $\mathcal{G}_0(s_b) = B_1$ (and therefore $B_1 
\subseteq \mathcal{G}_0(s_b)$ holds), and $s_r$ where
$\mathcal{G}_0(s_r) = G_1$ (and therefore $\mathcal{G}_0(s_r) \subseteq G_1$), where $s_b$ and $s_r$ are the states in $M_2$ referring to the
passenger being located at the blue and red depots, respectively. Planning therefore consists of  finding
a plan from $s_b$ to $s_r$ at level $M_2$; this is virtually trivial (there are only four states in $M_2$
and the state space is fully connected).  

\textbf{Example Query 2.} Query $Q_2$ has the start state set as before, but now specifies a goal depot (the yellow depot) for the taxi. 
 $B_2$ refers to all states where the passenger is at the blue depot and the 
taxi is at an unknown depot, but $G_2$ refers to a single state.  
$M_2$ contains a state that has the same grounding set as $B_2$,  but no
state in $M_2$  is a subset of $G_2$ because no state in $M_2$ specifies the location of the taxi. The agent therefore
cannot find a planmatch for $Q_2$ at level $M_2$. 

At  $M_1$ no single state is a superset of $B_2$, but the agent finds 
 a collection of states $s_j$, such that $\mathcal{G}_0(\cup_j s_j) = B_2$. It also
finds a single state with the same grounding as $G_2$. Therefore, it builds
a plan at level $M_1$ for each state in $s_j$. 

\textbf{Example Query 3.} In query $Q_3$, the taxi begins at the red depot and the passenger
at the blue depot, and its goal is to leave the passenger at grid location $(1, 4)$, with the taxi goal
location left unspecified.
The start set, $B_3$, refers to a single state, and the goal set, $G_3$, refers to the set of states
where the passenger is located at $(1, 4)$. 

Again the agent starts at $M_2$. $B_3$ is a subset of the grounding of the single state in $M_2$ where the passenger is at the blue depot but the
taxi is at an unknown depot. However, $G_3$ is \textit{not} a superset of
any of the states in $M_2$, since none contain any state where the passenger
is not at a depot.  Therefore the agent cannot plan for $Q_3$ at level $M_2$.

At level $M_1$, it again find a state that is a superset of $B_3$, but no state that is a subset of $G_3$---all
states in $M_1$ now additionally specify the position of the taxi and passenger, but like the states in
$M_2$ they all fix the location
of the passenger at a depot. All state groundings are in fact disjoint from the grounding of $G_3$. 
The agent must therefore resort to planning in $M_0$, and the hierarchy does not
help (indeed, it results in a performance penalty due to the compute time to rule out $M_1$ and $M_2$).

\section{Summary}

We have introduced a framework for building abstraction hierarchies 
by alternating skill- and representation-acquisition phases.
The framework is completely automatic except for the choice of skill acquisition algorithm, to which our formulation is agnostic.
The resulting hierarchies combine temporal and state abstraction to realize
efficient planning and learning in the multi-task setting.


\bibliographystyle{aaai}
\bibliography{skillsymloop-aaai}

\end{document}